\documentclass[a4paper]{article}

\usepackage[utf8]{inputenc}
\usepackage[T1]{fontenc}

\usepackage[a4paper,margin=1in]{geometry}

\usepackage[numbers]{natbib}
\usepackage[usenames,dvipsnames]{xcolor}
    \definecolor{DarkRed}{rgb}{0.368,0.097,0.078}
\definecolor{DarkBlue}{rgb}{0.2,0.2,0.6}

\usepackage[colorlinks = true,
    linkcolor = blue,
    anchorcolor = blue,
    citecolor = blue,
    filecolor = blue,
    urlcolor = DarkRed,
    pagebackref]{hyperref}

\usepackage{amsthm} 
\usepackage{mathtools,thmtools}
\usepackage{amsfonts,amsmath,amssymb}
\usepackage[capitalise]{cleveref}
\usepackage{times}
\usepackage{algorithm}
\usepackage{algpseudocode}
\usepackage{thm-restate}
\usepackage{tcolorbox}
\usepackage{lipsum}
\usepackage{enumitem}
\usepackage{bm}
\usepackage{xspace}
\usepackage{nicefrac}
\usepackage{dsfont}
\usepackage{float}
\usepackage{bbm}
\usepackage{mdframed}

\usepackage[color=green!25,prependcaption,textsize=tiny]{todonotes}

\usepackage{multirow}

\usepackage{enumitem}
\usepackage{bm}
\usepackage{xspace}
\usepackage{nicefrac}

\usepackage{dsfont}

\declaretheoremstyle[
	    spaceabove=\topsep, 
	    spacebelow=\topsep, 
	    bodyfont=\normalfont\itshape,
    ]{theorem}

\declaretheorem[style=theorem,name=Theorem]{theorem}

\declaretheoremstyle[
	    spaceabove=\topsep, 
	    spacebelow=\topsep, 
	    bodyfont=\normalfont,
    ]{definition}

\declaretheoremstyle[
        spaceabove=\topsep, 
        spacebelow=\topsep, 
        bodyfont=\normalfont,
        notefont=\normalfont\bfseries,
        notebraces={}{},
        qed=$\blacksquare$, 
    ]{proofstyle}
\declaretheorem[style=proofstyle,numbered=no,name=Proof]{proof}

\declaretheorem[style=theorem,sibling=theorem,name=Lemma]{lemma}
\declaretheorem[style=theorem,sibling=theorem,name=Corollary]{corollary}

\declaretheorem[style=theorem,numbered=no,name=Theorem]{theorem*}
\declaretheorem[style=theorem,numbered=no,name=Lemma]{lemma*}
\declaretheorem[style=theorem,numbered=no,name=Corollary]{corollary*}
\declaretheorem[style=theorem,numbered=no,name=Proposition]{proposition*}
\declaretheorem[style=theorem,numbered=no,name=Claim]{claim*}
\declaretheorem[style=theorem,numbered=no,name=Fact]{fact*}
\declaretheorem[style=theorem,numbered=no,name=Observation]{observation*}
\declaretheorem[style=theorem,numbered=no,name=Conjecture]{conjecture*}

\declaretheorem[style=definition,sibling=theorem,name=Definition]{definition}
\declaretheorem[style=definition,sibling=theorem,name=Remark]{remark}

\declaretheorem[style=definition,numbered=no,name=Definition]{definition*}
\declaretheorem[style=definition,numbered=no,name=Remark]{remark*}
\declaretheorem[style=definition,numbered=no,name=Example]{example*}
\declaretheorem[style=definition,numbered=no,name=Question]{question*}

\DeclareMathAlphabet{\mathbfsf}{\encodingdefault}{\sfdefault}{bx}{n}

\newcommand{\lr}[1]{\mathopen{}\left(#1\right)}

\newcommand{\lrbra}[1]{\mathopen{}\left[#1\right]}
\newcommand{\Lrbra}[1]{\mathopen{}\big[#1\big]}

\newcommand{\lrset}[1]{\mathopen{}\left\{#1\right\}}

\newcommand{\lrabs}[1]{\mathopen{}\left|#1\right|}

\usepackage{amsfonts,amsmath,amssymb}
\usepackage{mathtools}
\usepackage{float}
\usepackage{enumitem}
\usepackage{algorithm}
\usepackage{algpseudocode}

\newcommand{\cD}{\mathcal{D}}

\newcommand{\cO}{\mathcal{O}}

\newcommand{\cH}{\mathcal{H}}

\newcommand{\cX}{\mathcal{X}}
\newcommand{\cY}{\mathcal{Y}}

\newcommand{\cF}{\mathcal{F}}

\newcommand{\cG}{\mathcal{G}}

\newcommand{\hP}{\mathbb{P}}
\newcommand{\hI}{\mathbb{I}}

\newcommand{\categorial}{\mathrm{Categorial}}
\newcommand{\uniform}{\mathrm{Uniform}}

\newcommand{\fat}{\mathrm{fat}}

\newcommand{\med}{\mathrm{Median}}

\newcommand{\pseudo}{\mathrm{Pdim}}

\newcommand{\VC}{\mathrm{VC}}
\newcommand{\graphdim}{\mathrm{d}_\mathrm{G}}
\newcommand{\DS}{\mathrm{DS}}
\newcommand{\polylog}{\mathrm{Polylog}}

\newcommand{\hide}[1]{}

\newcommand{\R}{\mathbb{R}}

\newcommand{\X}{\mathcal{X}}
\newcommand{\F}{\mathcal{F}}
\newcommand{\Y}{\mathcal{Y}}
\newcommand{\K}{\mathcal{K}}

\renewcommand{\vec}[1]{\ensuremath{\text{{\bf\textrm{#1}}}}}

\newcommand{\set}[1]{\left\{#1\right\}}

\DeclarePairedDelimiter\norm{\lVert}{\rVert}  
\DeclarePairedDelimiter\abs{\lvert}{\rvert}

\newcommand{\beq}{\begin{eqnarray*}}
\newcommand{\eeq}{\end{eqnarray*}}
\newcommand{\beqn}{\begin{eqnarray}}
\newcommand{\eeqn}{\end{eqnarray}}

\newcommand{\argmin}{\mathop{\mathrm{argmin}}}

\newcommand{\paren}[1]{\left( #1 \right)}

\newcommand{\eps}{\varepsilon}
\newcommand{\bari}{{\hat\imath}}
\newcommand{\barj}{{\hat\jmath}}

\DeclarePairedDelimiter\inner{\langle}{\rangle}
\newcommand{\lagrangian}{\mathcal{L}}

\newcommand{\svm}{\operatorname{\texttt{SVM}}}

\title{Agnostic Sample Compression Schemes for Regression \footnote{Some of the results on linear regression presented in this paper (\Cref{subsec:impossible-exact-compression,subsec:exact-compression-l1-linfty}) previously appeared in the unpublished manuscript titled "Agnostic sample compression for linear regression" \cite{hanneke2018agnostic}.
}}

\author{
    Idan Attias \thanks{Department of Computer Science, Ben-Gurion University; \texttt{idanatti@post.bgu.ac.il}.} 
    \and  Steve Hanneke\thanks{Department of Computer Science, Purdue University; \texttt{steve.hanneke@gmail.com.}}
    \and  Aryeh Kontorovich\thanks{Department of Computer Science, Ben-Gurion University; \texttt{karyeh@cs.bgu.ac.il.}}
    \and  Menachem Sadigurschi\thanks{Department of Computer Science, Ben-Gurion University; \texttt{sadigurs@post.bgu.ac.il.}}
}

\begin{document}
\maketitle

\begin{abstract}
We obtain the first positive results for
bounded sample compression in the
agnostic regression setting with the $\ell_p$ loss, where $p\in [1,\infty]$.
We construct a generic \emph{approximate} sample compression scheme for real-valued function classes 
exhibiting exponential size in the fat-shattering dimension but independent of the sample size.
Notably, for linear regression, an \emph{approximate} compression of size linear in the dimension is constructed.
Moreover, for $\ell_1$ and $\ell_\infty$ losses, we can even exhibit an efficient \emph{exact} sample compression scheme of size linear in the dimension.
We further show that for every other $\ell_p$ loss, $p\in (1,\infty)$, 
there does not exist an exact agnostic compression scheme of bounded size. 
This refines and generalizes a negative result of
\citet*{david2016supervised} for the $\ell_2$ loss.
We close by posing general open questions: for agnostic 
regression with $\ell_1$ loss, 
does every 
function class admits an exact compression scheme of size 
equal to its pseudo-dimension? For the $\ell_2$ loss, 
does every function class admit an approximate compression scheme of polynomial size in the fat-shattering dimension?
These questions generalize Warmuth's  
classic sample compression conjecture 
for realizable-case classification 
\citep{warmuth2003compressing}.
\end{abstract}

\section{Introduction}\label{sec:intro}
Sample compression is a central problem in learning theory, whereby one seeks
to retain a ``small'' subset of the labeled sample that uniquely defines a
``good'' hypothesis.
Quantifying {\em small} and {\em good} specifies the different variants of the problem.
For instance, in the classification setting, 
taking {\em small} to mean ``constant size'' (i.e., depending only on the VC-dimension $d$ of
the concept class but not on the sample size $m$) and {\em good} to mean ``consistent with the sample''
specifies the classic realizable
sample compression problem
for VC classes.
The feasibility of 
the latter
was an open problem
between its being posed by
\citet*{Littlestone86relatingdata}
and its positive resolution by
\citet{moran2016sample},
with various intermediate steps in between
\citep*{floyd1989space,helmbold1992learning,DBLP:journals/ml/FloydW95,ben1998combinatorial,DBLP:journals/jmlr/KuzminW07,rubinstein2009shifting,rubinstein2012geometric,MR3047077,livni2013honest,moran2017teaching}.
A stronger form of this problem, where {\em small} means $\cO(d)$ (or even exactly $d$), remains open \citep{warmuth2003compressing}.

\citet*{david2016supervised}
recently 
generalized the definition of \emph{compression scheme}
to the agnostic case, where it is required that the 
function reconstructed from the compression set obtains 
an average loss on the full data set nearly as small 
as the function in the class that minimizes this quantity.
In \Cref{rem:compress-def}, we give a strong 
motivation for this criterion by arguing an equivalence 
to the generalization ability of the compression-based
learning algorithm. 
Under this definition, \citet{david2016supervised} 
extended the realizable-case result
for VC classes to cover the agnostic
case as well:
a bounded-size compression scheme
for the former
implies such a scheme (in fact, of the same size)
for the latter. They also generalized
from binary to multiclass concept families,
with the graph dimension in place of VC-dimension.
Proceeding to real-valued function classes, 
\citet{david2016supervised} came to 
a starkly negative conclusion:
they established
that there is \emph{no} 
constant-size 
exact agnostic sample compression scheme 
for linear functions under the $\ell_2$ loss.
({\em Realizable} linear regression in $\R^d$ trivially
admits sample compression of size $d+1$, under any loss, 
by selecting a minimal subset that spans the data.)

\paragraph{Main results.}  
We are the first to construct bounded sample compression schemes for agnostic regression with $\ell_p$ loss, $p\in [1,\infty]$. \Cref{table:summary} summarizes our contributions in the context of previous results. 
We refer to an $\alpha$-approximate compression as one where the function reconstructed from the compression set achieves an average error at most $\alpha$ compared to the optimal function in the class. We consider the sample compression to be exact when we precisely recover this error. See \Cref{eq:exact-comression,eq:approximate-compression} for the precise definitions.

Our approach begins with proposing a boosting method (\Cref{alg:compression}) to construct an $\alpha$-\emph{approximate} sample compression scheme for agnostic $\ell_p$ regression, within function classes characterized by a finite fat-shattering dimension. The scheme has a size of $\Tilde{\cO}\lr{\fat\lr{\cF,c\alpha/p}\fat^*\lr{\cF,c\alpha/p}}$\footnote{$\Tilde{\cO}$ hides polylogarithmic factors in the specified expression.}, for some numerical constant $c>0$, as established by \Cref{thm:general-classes-compression}. Here, $\fat\lr{\cF,c\alpha/p}$ represents the fat-shattering dimension of function class $\cF$ at scale $c\alpha/p$, and $\fat^*$ is the dimension of the dual-class, which is finite as long as the dimension of the primal class is finite and can be at most exponentially larger, see \Cref{eq:dual-fat}. Notably, our compression size is independent of the sample size.
A major open question is how to improve the exponential dependence in the dimension, even in the realizable binary classification setting \citep{warmuth2003compressing}. 
%
 While such an approximate compression has been previously acknowledged in realizable regression \cite{hanneke2019sample}, and exact compression in agnostic binary classification \cite{david2016supervised}, in \Cref{sec:compression-general-classes} we delve into the details of our techniques and elucidate why methods previously suggested fall short in addressing agnostic regression.

We proceed with exploring linear regression. 
The negative result of \citet{david2016supervised} regarding the impossibility of achieving an \emph{exact} compression for linear regression with the $\ell_2$ (squared) loss
raises a general doubt over whether exact sample compression 
is ever a viable approach to agnostic learning of real-valued functions.
We address this concern by proving that, 
if we replace the $\ell_2$ loss with the $\ell_1$ or $\ell_\infty$ loss, 
then there \emph{is} a simple exact agnostic compression scheme  
of size $d+1$ for $\ell_1$ linear regression and $d+2$ for $\ell_\infty$ in $\R^d$, see \Cref{thm:ell1,thm:ellinfty}.  This is 
somewhat surprising, given the above negative result for the 
$\ell_2$ loss. 
Computationally, our compression schemes for $\ell_1$ and $\ell_\infty$ amount to solving a polynomial (in fact, linear)
size linear program.

We then propose \Cref{alg:lin-regression-lp} for an $\alpha$-approximate sample compression for $\ell_p$ linear regression   of size $\cO\lr{d\log(p/\alpha)}$, where $p\in (1,\infty)$, see \Cref{thm:lin-regression-approx-comprssion}. Roughly speaking, we reduce the problem to realizable binary classification with linear functions. Our approach involves introducing a discretized dataset on which the optimal solution of Support Vector Machine (SVM) pointwise approximates an optimal regressor on the original dataset.
We complement this result by showing that $p\in \lrset{1,\infty}$ are the \emph{only two} $\ell_p$ losses for which a constant-size compression scheme exists (\Cref{thm:our-lb}), generalizing the argument of \citet{david2016supervised}. 

These appear to be the first positive results for bounded agnostic sample
compression for real-valued function classes.
We close by posing intriguing open questions generalizing our result 
to arbitrary function classes: under the $\ell_1$ loss, 
does \emph{every} function class admit an exact agnostic compression 
scheme of size equal to its pseudo-dimension?  under the $\ell_2$ loss, does \emph{every} function class admit an approximate agnostic compression of size equal to its fat-shattering dimension? We argue that this 
represents a generalization of Warmuth's classic 
sample compression problem, which asks whether every 
space of classifiers admits a compression scheme of size VC-dimension 
in the realizable case.

\paragraph{Related work.}
Sample compression scheme is a classic technique for proving generalization bounds, introduced by \citet{littlestone1986relating,floyd1995sample}. These bounds proved to be useful in numerous learning settings, particularly when the uniform convergence property does not hold or provides suboptimal rates,  such as binary classification \cite{graepel2005pac,moran2016sample,bousquet2020proper}, multiclass classification \cite{daniely2015multiclass,daniely2014optimal,david2016supervised,brukhim2022characterization}, regression \cite{hanneke2019sample,attias2023optimal}, active learning \cite{wiener2015compression}, density estimation \cite{ashtiani2020near}, adversarially robust learning \cite{montasser2019vc,montasser2020reducing,montasser2021adversarially,montasser2022adversarially,attias2022characterization,attias2023adversarially}, learning with partial concepts \cite{alon2022theory}, and showing Bayes-consistency for nearest-neighbor methods \cite{gottlieb2014near,kontorovich2017nearest}.
As a matter of fact, compressibility and learnability are known to be equivalent for general learning problems \cite{david2016supervised}.
A remarkable result by \citet{moran2016sample} showed that VC classes enjoy a sample compression that is independent of the sample size. 

\citet{david2016supervised} introduced sample compression in the context of regression. They showed that an exact compression scheme for $\ell_2$ agnostic linear regression requires a linear growth relative to the sample size.  Additionally, they showed that it is feasible to have an $\alpha$-approximate compression 
for zero-dimensional linear regression with a size of $\log(1/\alpha)/\alpha$. In a broader sense, they established the equivalence between learnability and the presence of an approximate compression in regression.

\citet{hanneke2019sample}
showed how to convert {\em consistent} real-valued learners into constant-size
(i.e., independent of sample size) efficiently computable approximate compression schemes 
for the realizable (or nearly realizable) regression with the $\ell_\infty$ loss.
This result was obtained via a weak-to-strong boosting procedure,
coupled with a generic construction of weak learners out of abstract regressors.
The {\em agnostic} variant of this problem remains open in its full generality.

\citet{ashtiani2020near} adapted the notion of a compression scheme
to the distribution learning problem.
They showed
that if a class of distributions admits
robust compressibility
then
it is agnostically learnable.

\begin{table}[H]
    \footnotesize
    \begin{center}
    \bgroup
    \setlength{\arrayrulewidth}{0.2mm} 
    \setlength{\tabcolsep}{4.4pt} 
    \def\arraystretch{2} 
        \begin{tabular}{|c|c|c|c|}
            \hline
            \textbf{Problem Setup} 
            & \textbf{Compression Type} 
            & \textbf{Compression Size} 
            & \textbf{Reference}
            \\
            \hline
            \hline
            Realizable/Agnostic Binary Classification
            & Exact
            & $\cO\lr{\VC \cdot \VC^*}$ 
            & \cite{moran2016sample,david2016supervised}
            \\
            \hline
            \multirow{3}{*}{\centering{Realizable/Agnostic Multiclass Classification}}
            & \multirow{3}{*}{\centering{Exact}}
            & $\cO\lr{\graphdim \cdot \graphdim^*}$
            & \cite{david2016supervised}
            \\
            & 
            & $\cO\lr{\DS^{1.5}\cdot\polylog\lr{m}}$
            & \cite{brukhim2022characterization}
            \\
            &
            & $\Omega\lr{\log\lr{m}^{1-o(1)}}$
            & \cite{pabbaraju2023multiclass} 
            \\
            \hline
            %
            Realizable $\ell_\infty$ Regression
            & $\alpha$-Approximate
            & $\cO\lr{\fat_{c\alpha}\cdot \fat^*_{c\alpha}\cdot\polylog\lr{\fat_{c\alpha},\fat^*_{c\alpha},\frac{1}{\alpha}}}$
            & \cite{hanneke2019sample}
            \\
            \hline
            Agnostic $\ell_p$  Regression: $p\in\lr{1,\infty}$
            & \multirow{2}{*}{\centering{$\alpha$-Approximate}}
            & $\cO\lr{\fat_{c\alpha}\cdot \fat^*_{c\alpha}\cdot\polylog\lr{\fat_{c\alpha},\fat^*_{c\alpha},p,\frac{1}{\alpha}}}$
            & \multirow{2}{*}{\centering{This work}}
            \\
            Agnostic $\ell_p$  Regression: $p\in\lrset{1,\infty}$
            &
            & $\cO\lr{\fat_{c\alpha}\cdot \fat^*_{c\alpha}\cdot\polylog\lr{\fat_{c\alpha},\fat^*_{c\alpha},\frac{1}{\alpha}}}$  
            & 
            \\
            \hline
            %
            Agnostic $\ell_p$ Linear Regression: $p\in\lrset{1,\infty}$
            & Exact
            & $\cO\lr{d}$ 
            & This work 
            \\
            Agnostic $\ell_p$ Linear Regression: $p\in\lr{1,\infty}$
            & $\alpha$-Approximate
            & $\cO\lr{d\cdot\log\lr{\frac{p}{\alpha}}}$
            & This work
            \\
            Agnostic $\ell_2$ Linear Regression
            & Exact
            & $\Omega\lr{m}$  
            & \cite{david2016supervised}
            \\
            Agnostic $\ell_p$ Linear Regression: $p\in \lrbra{1,\infty}$
            & Exact
            & $\Omega\lr{\log\lr{m}}$  
            & This work
            \\
            \hline
        \end{tabular}
    \egroup
    \end{center}
    \caption{\textbf{Sample compression schemes for classification and regression.}
    We denote the sample size by $m$, $c>0$ is a numerical constant. The $o(1)$ term vanishes as $m\rightarrow \infty$. 
    \textbf{(i) Binary Classification:} $\VC$ is the Vapnik-Chervonenkis dimension that characterizes realizable and agnostic learnability. Any dimension with $(\cdot)^*$ denotes the dimension of the dual-class.
    \textbf{(ii) Multiclass Classification: }$\graphdim$ is the Graph-dimension and $\DS$ is the Daniely-Shwartz dimension. For a finite set of labels, both dimensions characterize realizable and agnostic learnability.
    For an infinite set, only the finiteness of the DS dimension is equivalent to learnability. There exist learnable function classes with infinite graph dimension and finite DS dimension.
    (\textbf{iii) Regression:} $\fat_{c\alpha}$ is the fat-shattering dimension at scale $c\alpha$. A function class is agnostically learnable in this setting if and only if the fat-shattering dimension is finite for any scale. However, in the realizable case, there are learnable classes with infinite fat-shattering dimension.
    We comment that the results in \cite{hanneke2019sample} are stated for $\ell_\infty$, but still hold for any $\ell_p$ (with extra polylog factors in $p$) due to Lipschitzness of this loss.
    \textbf{(iv) Linear Regression:} $d$ is the vector space dimension. 
    We refer to \Cref{sec:open-problems} for open problems.
    }
    \label{table:summary}
\end{table}

\section{Preliminaries}
We denote $[m]:=\set{1,\ldots,m}$.
Let $\F\subseteq\Y^\X$ be a hypothesis class. The $\ell_p$ loss incurred by a hypothesis $f\in\F$ on $\lr{x,y}$ is given by  $\lr{x,y}\mapsto \lrabs{f(x)-y}^p$, where $p\in \lrbra{1,\infty}$.
For $p\in[1,\infty)$,
the
loss incurred by a hypothesis $f\in\F$
on a labeled sample
$S = \lrset{(x_i,y_i):i\in [m]}$
is given by
\begin{align}\label{eq:Lp}
  L_p(f,S):=\frac1m\sum_{i=1}^m|f(x_i)-y_i|^p,  
\end{align}
while for $p=\infty$,
\begin{align}\label{eq:Linf}
L_\infty(f,S):=\max_{1\le i\le m}|f(x_i)-y_i|.   
\end{align}
\begin{remark}
    The $\ell_p$ regression objective is typically written 
    without taking the $p$th root    
so as to facilitate
optimization algorithms.
    As we avoid taking the $p$-th root, the resulting $p$-norm formulation does not directly converge to $\ell_\infty$ as $p$ approaches infinity. Consequently, our $\ell_p$ results explicitly depend on $p$, similar to results in the literature.
\end{remark}
Now let us introduce a formal definition of sample compression, 
and a criterion we require of any valid \emph{agnostic compression scheme}.
Following the definition, we provide a strong motivation for this criterion 
in terms of an equivalence to the generalization ability of the learning algorithm
under general conditions.

\paragraph{Approximate and exact sample compression schemes.}
Following
\citet{david2016supervised},
we define
a {\em selection scheme} $(\kappa,\rho)$ for a hypothesis
class $\F\subset\Y^\X$
is defined as follows.
A $k$-{\em selection} function $\kappa$
maps sequences $\lrset{(x_1,y_1),\ldots,(x_m,y_m)}\in\bigcup_{\ell\ge1}\lrset{\X\times\Y}^\ell$
to elements in
$\K=\bigcup_{\ell\le k'}\lrset{\X\times\Y}^\ell
\times
\bigcup_{\ell\le k''}\set{0,1}^\ell$,
where $k'+k''\le k$.
A {\em reconstruction} is a function $\rho:\K\to\Y^\X$.
We say that $(\kappa,\rho)$ is a $k$-size agnostic \emph{exact} sample compression
scheme for $\F$
if
$\kappa$ is a $k$-selection and
for all
$S = \lrset{(x_i,y_i):i\in [m]}$,
$f_S := \rho(\kappa(S))$
achieves $\F$-competitive empirical loss:
\begin{align}\label{eq:exact-comression}
    L_p(f_S,S) &\le \inf_{f\in\F}L_p(f,S).
\end{align}

We also define a relaxed notion of agnostic $\alpha$-\emph{approximate} sample compression in which $f_S$ should satisfy
\begin{align}\label{eq:approximate-compression}
   L_p(f_S,S) &\le \inf_{f\in\F}L_p(f,S)+\alpha. 
\end{align}

In principle, the \emph{size} $k$ of an agnostic compression scheme may 
depend on the data set size $m$, in which case we may denote this dependence by $k(m)$.
However, in this work we are primarily interested in the case when $k(m)$ is \emph{bounded}: 
that is, $k(m) \leq k$ for some $m$-independent value $k$.
Note that the above definition is fully general, in that it defines a 
notion of agnostic compression scheme for \emph{any} function class $\F$ 
and loss function $L$, though in the present work we focus on $L_p$ loss for $1 \leq p \leq \infty$.

\begin{remark}
  \label{rem:compress-def} 
At first, it might seem unclear why this is 
an appropriate generalization of sample compression 
to the agnostic setting.  To see that it is so, 
we note that one of the main interests in 
sample compression schemes is their 
ability to \emph{generalize}: that is, to achieve 
\emph{low excess risk} under a \emph{distribution} $P$ on $\X\times\Y$ 
when the data $S$ are sampled iid according to $P$
\citep*{Littlestone86relatingdata,DBLP:journals/ml/FloydW95,DBLP:journals/ml/GraepelHS05}.
Also, as mentioned, in this work we are primarily interested in sample compression schemes 
that have \emph{bounded size}: $k(m) \leq k$ for an $m$-independent value $k$.
Furthermore, we are also focusing on the most-general case, where this size bound
should be independent of everything else in the scenario, 
such as the data $S$ or the underlying distribution $P$.
Given these interests, we claim that the above definition 
is essentially the only reasonable choice.  More specifically, for $L_p$ loss with $1 \leq p < \infty$, 
any compression scheme with $k(m)$ bounded such that its 
expected excess risk under any $P$ converges to $0$ as $m \to \infty$ 
necessarily satisfies the above condition (or is easily converted into one that does).  
To see this, note that for any data set $S$ for which such a compression scheme 
fails to satisfy the above $\F$-competitive empirical loss criterion, 
we can define a distribution $P$ that is simply uniform on $S$, 
and then the compression scheme's selection function would be choosing 
a bounded number of points from $S$ and a bounded number of bits, 
while guaranteeing that excess risk under $P$ approaches $0$, 
or equivalently, excess empirical loss approaches $0$.  To make this 
argument fully formal, only a slight modification is needed, to handle 
having multiple copies of points from $S$ in the compression set; 
given that the size is bounded, these repetitions can be encoded in 
a bounded number of extra bits, so that we can stick to strictly distinct 
points in the compression set.

In the converse direction, we also note that any bounded-size 
agnostic compression scheme (in the sense of the above definition) 
will be guaranteed to have excess risk under $P$ converging to $0$ as $m \to \infty$, 
in the case that $S$ is sampled iid according to $P$, for losses $L_p$ with $1 \leq p < \infty$, 
as long as $P$ guarantees that $(X,Y) \sim P$ has $Y$ bounded (almost surely).
This follows from classic arguments about the generalization ability of compression schemes, 
which includes results for the agnostic case
\citep*{DBLP:journals/ml/GraepelHS05}.
For unbounded 
$Y$ one cannot, in general, obtain distribution-free generalization bounds.
However, one can still obtain generalization under certain broader restrictions 
(see, e.g., \citealp{DBLP:journals/jacm/Mendelson15} and references therein).
The generalization problem becomes more subtle for the $L_\infty$ loss:
this cannot be expressed as a sum of pointwise losses and 
there are no standard techniques for bounding the deviation of the sample risk from the true risk.
One recently-studied guarantee achieved by minimizing empirical $L_\infty$ loss 
is a kind of ``hybrid error'' generalization, developed in 
\citet[Theorem 9]{hanneke2019sample}.
We refer the interested reader to that work for the details of those results, 
which can easily be extended to apply to our notion of an agnostic compression scheme.
\end{remark}

\paragraph{Complexity measures.}
Let $\cF\subseteq[0,1]^\cX$ and $\gamma > 0$. We say that $S = \{x_1, \ldots, x_m\} \subseteq \cX$ is $\gamma$-shattered by $\cF$ if there exists a witness $r = (r_1, \ldots,r_m) \in [0,1]^m$ such that for each $\sigma = (\sigma_1, \ldots, \sigma_m) \in \{-1, 1\}^m$ there is a function $f_{\sigma} \in \cF$ such that
\[
\forall i \in [m] \; 
\begin{cases}
	f_{\sigma}(x_i) \geq r_i + \gamma, & \text{if $\sigma_i = 1$}\\
    f_{\sigma}(x_i) \leq r_i - \gamma, & \text{if $\sigma_i = -1$.}
 \end{cases}
\]
The \emph{fat-shattering dimension} of $\cF$ at scale $\gamma$, denoted by $\fat(\cF,\gamma)$, is 
the cardinality of the largest set of points in $\cX$ that can be $\gamma$-shattered by $\cF$.
This parametrized variant of the Pseudo-dimension \citep{alon1997scale}
was first proposed by \citet{kearns1994efficient}.
Its key role in learning theory
lies in characterizing the PAC learnability
of real-valued function classes
\citep{alon1997scale,bartlett1998prediction}.
We also define the dual dimension. Define the dual class $\cF^* \subseteq [0,1]^\cF$ of $\cF$ as the set of all functions $g_w: \cF \rightarrow [0,1]$ defined by $g_w(f) = f(w)$. If we think of a function class as a matrix whose rows and columns are indexed by functions and points, respectively, then the dual class is given by the transpose of the matrix. 
The \emph{dual fat-shattering dimension} at scale $\gamma$, is defined as the fat-shattering at scale $\gamma$ of the dual-class and denoted by $\fat^*\lr{\cF,\gamma}$.
We have the following bound due to \citet[Corollary 3.8 and inequality 3.1]{kleer2021primal},
\begin{align}\label{eq:dual-fat}
   \fat^*\lr{\cF,\gamma} 
   \lesssim 
   \frac{1}{\gamma}2^{\fat\lr{\cF,\gamma/2}+1}.
\end{align}
\section{Approximate Agnostic Compression for Real-Valued Function Classes}\label{sec:compression-general-classes}
In this section, we construct an approximate compression scheme for all real-valued function classes that are agnostically PAC learnable, that is, classes with finite fat-shattering dimension at any scale \citep{alon1997scale,bartlett1998prediction}. We prove the following main result. 

\begin{theorem}[Approximate compression for agnostic regression]\label{thm:general-classes-compression}
Let $\cF \subseteq [0,1]^\cX$, $S = \lrset{(x_i,y_i): i\in [m]}\subseteq \cX \times [0,1]$, an approximation parameter $\alpha\in [0,1]$, a weak learner parameter $\beta \in (0,1/2]$, and $\ell_p$ loss where $p\in [1,\infty]$.
By setting \Cref{alg:compression} with $T \leftarrow \cO\lr{\frac{1}{\beta^2}\log(m)}$ and 
%
\[
\begin{dcases}
    d \leftarrow \Tilde{\cO}\lr{\fat\lr{\cF,c\alpha/p}
    },
    n \leftarrow
        \Tilde{\cO}\lr{\frac{\fat^*\lr{\cF,c\alpha/p}}{\beta^2}
        }, p\in [1,\infty)\\
   d\leftarrow
   \Tilde{\cO}\lr{\fat\lr{\cF,c\alpha}
   },
    n \leftarrow
        \Tilde{\cO}\lr{\frac{\fat^*\lr{\cF,c\alpha}}{\beta^2}
        }, p=\infty,
\end{dcases}
\]  
we get an $\alpha$-approximate sample compression scheme of size
\[
\begin{dcases}
	\Tilde{\cO}\lr{\frac{1}{\beta^2}\fat\lr{\cF,c\alpha/p}\fat^*\lr{\cF,c\alpha/p}
 }, p\in [1,\infty)\\
    \Tilde{\cO}\lr{\frac{1}{\beta^2}\fat\lr{\cF,c\alpha}\fat^*\lr{\cF,c\alpha}
    }, p=\infty,
 \end{dcases}
\]
for some universal constant $c>0$. Recall that the dual fat-shattering is at most exponential in the primal dimension, see \cref{eq:dual-fat}. $\Tilde{\cO}(\cdot)$ hides polylogarithmic factors of $\lr{\fat,\fat^*,p,1/\alpha,1/\beta}$.
\end{theorem}
Our algorithm incorporates a boosting approach for real-valued functions. Therefore, we need a definition of weak learners in this context. 
\begin{definition}[Approximate weak real-valued learners]\label{def:weak-learner}
Let $\beta\in(0,\frac{1}{2}]$, $\alpha\in (0,1)$.
We say that $g: \cX \rightarrow [0,1]$ is an \emph{approximate} $(\alpha,\beta)$-weak learner, with respect to $P$ and a target function $f^\star\in\cF$ if
\begin{equation*}
\hP_{(x,y)\sim P}\lrset{(x,y):\lrabs{g(x)-y}>\lrabs{f^\star(x)-y}+\alpha}\leq
\frac{1}{2}-\beta.
\end{equation*}
\end{definition}
This notion of a weak learner must be formulated carefully. 
For example, taking a learner guaranteeing absolute loss at most $\frac{1}{2}-\beta$ is known to not be strong enough for boosting to work, see the discussion in \citet[Section 4]{hanneke2019sample}. On the other hand, by making the requirement too strong (for example, $\mathrm{AdaBoost.R}$ in \citet{freund1997decision}), then the sample complexity of weak learning will be high that weak learners cannot be expected to exist for certain function classes. 
We can now present the main algorithm. 
\begin{algorithm}[]
    \caption{Approximate Agnostic Sample Compression for $\ell_p$ Regression, $p \in[1,\infty]$}\label{alg:compression}
    \textbf{Input}: $\cF\subseteq [0,1]^\cX$, $S=\lrset{\lr{x_i,y_i}: i\in [m]}\subseteq \cX \times [0,1]$.\\
    \textbf{Parameters}: Approximation parameter $\alpha\in (0,1)$, weak learner parameter $\beta \in (0,1/2]$, weak learner sample size $d\geq 1$, sparsification parameter $n\geq 1$, number of boosting rounds $T\geq 1$, loss parameter $p\in \lrbra{1,\infty}$.\\
    \textbf{Initialize}: $P_1 \leftarrow \uniform(S)$.                 
    %
    %
    \begin{enumerate}[leftmargin=0.35cm]
    \item[$\triangleright$]\textcolor{DarkBlue}{Find an optimal function in $f^\star\in\cF$. Our goal is to construct a function that pointwise approximates $f^\star$ on $S$}
    \item Compute:
        \begin{enumerate}
            \item 
            $f^\star\leftarrow \argmin_{f\in\cF}L_p(f,S)$ (defined in \cref{eq:Lp,eq:Linf}).
            %
            \item $\psi(x,y) \leftarrow|f^\star(x)-y|,\; \forall (x,y)\in S$. 
        \end{enumerate}
    \item[$\triangleright$]\textcolor{DarkBlue}{Median boosting for real-valued functions}
    \item For $t=1,\ldots,T$:
        \begin{enumerate}
            \item Get an $(2\alpha,\beta)$-approximate weak learner $\hat{f}_t$ with respect to distribution $P_t$: \\
            Find a multiset $S_t\subset S$ of $d$ points such that for any ${f}\in\cF$ with $\lrabs{{f}(x)-y}\leq \psi(x,y)+\alpha$ $\forall{(x,y)\in S_t}$,
it holds that $\hP_{(x,y)\sim P_t}\lrset{(x,y): \lrabs{f(x)-y} > \psi(x,y)+2\alpha}\leq 1/2- \beta$. ($S_t$ exists from \Cref{thm:generalization-interpolation-cutoffs}).
            \item For $i=1,\ldots,m$:\\
                Set $w^{(t)}_i \leftarrow 1-2\hI\Lrbra{\lrabs{\hat{f}_t(x_i)-y_i}>\psi(x_i,y_i)+2\alpha}.$
                \item Set $\alpha_t \leftarrow
                \frac{1}{2}\log\lr{\frac{\lr{1-\beta}\sum^m_{i=1} P_t\lr{x_i,y_i}\hI\lrbra{w_i^{(t)}=1}}{\lr{1+\beta}\sum^m_{i=1} P_t\lr{x_i,y_i}\hI\lrbra{w_i^{(t)}=-1}}}.$
                \item 
                If $\alpha_t=\infty$: \\
                return $T$ copies of $\hat{f}_t$, $\lr{\alpha_1=1,\ldots,\alpha_T=1}$, $S_t$.\\
                Else: \\
                $ P_{t+1}(x_i,y_i) \leftarrow  P_{t}(x_i,y_i) \frac{\exp\lr{-\alpha_t w_i^{t}}}{\sum^m_{j=1} P_t\lr{x_j,y_j}\exp\lr{-\alpha_t w_j^{t}}}.$
        \end{enumerate}
    \item[$\triangleright$]\textcolor{DarkBlue}{Sparsifying the weighted ensemble $\lrset{\hat{f}_i}^T_{i=1}$ returned from boosting via sampling}
    \item 
    Repeat:
        \begin{enumerate}
            \item Sampling: \\
            $\lr{J_1,...\,,J_n}\sim \categorial\lr{\frac{\alpha_1}{\sum^T_{s=1}\alpha_s},...\,,\frac{\alpha_T}{\sum^T_{s=1}\alpha_s}}^n$.
            \item Let $\Tilde{\cF}=\lrset{f_{J_1},\ldots,f_{J_n}}$.
            \item Until $\forall (x,y)\in S:\\
          \lrabs{\lrset{f\in\Tilde{\cF}:\lrabs{f(x)-y}>\psi(x,y)+3\alpha}}<n/2$.
        \end{enumerate}
    \end{enumerate}
    \textbf{Compression:} Multisets $S_{J_1},\ldots, S_{J_n}$ and cut-offs $\psi\vert_{S_{J_1}},\ldots,\psi\vert_{S_{J_n}}$ corresponding to the weak learners in $\Tilde{F}$.
    \\
    \textbf{Reconstruction:} Reconstruct weak learners $f_{J_i}$ from $S_{J_i}$ and $\psi\vert_{S_{J_i}}$, $i\in [n]$, and output their median $\med\lr{f_{J_1},\ldots,f_{J_n}}$.
\end{algorithm}
\paragraph{The challenges beyond realizable regression and agnostic classification.}
 There is a crucial difference from previous boosting algorithms for real-valued used by \citet{kegl2003robust,hanneke2019sample} in the realizable case. In our approach, the cut-offs $\psi(x,y)$ are allowed to vary across different points, in contrast to a fixed cut-off applied uniformly across all points. This flexibility enables us to address the agnostic setting, wherein the loss of an optimal minimizer may differ across various points in the sample. 
To prove the existence of weak learners we are required to have a generalization theorem that is compatible with changing cut-offs, see \Cref{thm:generalization-interpolation-cutoffs}. A similar generalization result was used in the context of adversarially robust learning \citep{attias2023adversarially}.

The compression approach for agnostic binary classification, as discussed in \citep{david2016supervised}, encounters a similar challenge. In this method, our initial emphasis is on identifying the points correctly classified by an optimal function in the class. Subsequently, we apply compression techniques for realizable classification. 
However, in regression, discarding points where the optimal function makes mistakes is not feasible, given that the loss is not strictly zero-one. Instead, we utilize the entire sample, targeting the error for each point and constructing a function with a similar approximated error on each point. 

\paragraph{Proof overview.}
First, we show that the returned output of \Cref{alg:compression} is a valid compression. Then we bound the size of this compression.

\textit{Approximate compression correctness.}
In step 1, we compute some $f^\star\in\cF$ the minimizes the empirical $\ell_p$ error on the sample $S$,
$$f^\star
\leftarrow 
\argmin_{f\in\cF} L_p(f,S),$$
as defined in \cref{eq:Lp,eq:Linf}.
Let $\psi: \cX \times \cY \rightarrow [0,1]$ be the $\ell_1$ loss of $f^\star$ on each point in $S$, $$\psi(x,y) \leftarrow|f^\star(x)-y|,\; \forall (x,y)\in S.$$ 
 
In step 2, we implement a boosting algorithm, following \Cref{def:weak-learner} of weak learners.
By using \Cref{thm:generalization-interpolation-cutoffs} with  $\delta=1/3$ and $\eps=1/2-\beta$. For any distribution $P_t$ on $S$, upon receiving an i.i.d. sample $S_t\subseteq S$ from $P_t$ of size 
$$d= \cO\lr{\fat\lr{\cF,\alpha/8}\log^2\lr{\frac{\fat\lr{\cF,\alpha/8}}{\alpha(1/2-\beta)}}},$$ 
with probability $2/3$ over sampling $S_t$ from $P_t$, 
for any ${f}\in\cF$ satisfying $\forall (x,y)\in S_t:\; \lrabs{{f}(x)-y}\leq \psi(x,y)+\alpha$, 
it holds that $$\hP_{(x,y)\sim P_t}\lrset{(x,y): \lrabs{f(x)-y}>\psi(x,y)+2\alpha}\leq \frac{1}{2}-\beta.$$
That is, such a function is an approximate $\lr{2\alpha,\beta}$-weak learner for $P_t$ and $f^\star$. Since this holds with probability $2/3$, there must be such $S_t\subseteq S$. In order to construct an approximate $(2\alpha,\beta)$-weak learner $\hat{f}_t$, we need to find ${f}\in\cF$ such that $\forall (x,y)\in S_t:\; \lrabs{{f}(x)-y}\leq \psi(x,y)+\alpha$, and so the weak learner can be encoded by $S_t$ of size $d$ and the set of cut-offs $\psi(x,y)\in [0,1]$ for all $(x,y)\in S_t$. We encode only approximations of the cut-offs to keep the compression size bounded (see the paragraph about the compression size below).
For $T=\cO\lr{\frac{1}{\beta^2}\log(m)}$ rounds of boosting, \Cref{lem:boosting} guarantees that for all $(x,y)\in S$ the output of the boosting algorithm satisfies
\begin{align*}
\lrabs{
\med\lr{\hat{f}_1,\ldots,\hat{f}_T;\alpha_1,\ldots,\alpha_T}\lr{x}-y
} 
&\leq 
\psi(x,y)+2\alpha.
\end{align*}
Finally, we use sampling to reduce the number of hypotheses in the ensemble from $\cO\lr{\frac{1}{\beta^2}\log(m)}$ to size that is independent of $m$. \Cref{lem:sparsification} implies that the sparsification method in Step 3 ensures that we can sample  $$n=\cO\lr{\fat^*\lr{\cF,c\alpha}\log^2\lr{\fat^*\lr{\cF,c\alpha}/\alpha}}$$ such that for all $(x,y)\in S$
    $$
     \lrabs{\med\lr{f_{J_1}(x),\ldots,f_{J_n}(x)}-y} \leq \psi(x,y)+3\alpha,
    $$
where $c>0$ is an absolute constant.
By rescaling $3\alpha$ to $\alpha$, this proves the $\ell_1$ and $\ell_\infty$ losses.
For $p\in(1,\infty)$, we use the Lipschitzness of the $\ell_p$ loss and rescale the approximate parameter accordingly. 
We constructed a function $h$ with $\lrabs{h(x)-y} \leq \psi(x,y)+\alpha$ for any $(x,y)\in S$, which implies
\begin{align*}
\lrabs{h(x)-y}^p 
&\overset{(i)}{\leq} 
\lr{\lr{\psi(x,y)}+\alpha}^p 
&\overset{(ii)}{\leq}
\psi(x,y)^p+p\alpha,
\end{align*}
and that will finish the proof. (i) Follows by just raising both sides to the power of $p$.
(ii) Follows since the function $x\mapsto \lrabs{x-y}^p$ is $p$-Lipschitz for $(x-y)\in [0,1]$, and so
\begin{align*}
    \lrabs{\lr{\psi(x,y)+\alpha}^p - \psi(x,y)^p}
    &\leq
    p\lrabs{\psi(x,y)+\alpha - \psi(x,y)}
    \\
    &\leq
    p\alpha.
\end{align*}
By rescaling $p\alpha$ to $\alpha$, we get 
\begin{align*}
        \lrabs{\med\lr{f_{J_1}(x),\ldots,f_{J_n}(x)}-y}^p 
        &\leq 
        \psi(x,y)^p+\alpha
        ,
\end{align*}
where 
$$n=\Theta\lr{\frac{1}{\beta^2}\fat^*\lr{\cF,c\alpha/p}\log^2\lr{\frac{p\,\fat^*\lr{\cF,c\alpha/p}}{\alpha}}},$$
and 
$$d= \cO\lr{\fat\lr{\cF,c\alpha/p}\log^2\lr{\frac{p\,\fat\lr{\cF,c\alpha/p}}{\alpha(1/2-\beta)}}}.$$ 
We proved the correctness of an $\alpha$-approximate compression
$$L_p(\med\lr{f_{J_1},\ldots,f_{J_n}},S) 
\leq 
\inf_{f\in\cF} L_p(f,S) + \alpha.$$

\textit{Approximate compression size.}
Each weak learner is encoded by a multiset $S'\subseteq S$ of size $d$ and is constructed by computing some $f'\in \cF$ that solves the  constrained optimization 
$$ \lrabs{{f}'(x)-y}\leq \psi(x,y), \; \forall (x,y)\in S'.$$
We encode each $\psi(x,y)$ by some approximation $\Tilde{\psi}(x,y)$, such that $\lrabs{\Tilde{\psi}(x,y)-\psi(x,y)}\leq \alpha$, by discretizing $[0,1]$ to buckets of size $1/\alpha$, and each $\psi(x,y)$ is rounded down to the closest value $\Tilde{\psi}(x,y)$. Each approximation requires to encode $\log\lr{1/\alpha}$ bits, and so each learner encodes $d\log\lr{1/\alpha}$ bits and $d$ samples. We have $n$ weak learners, and the compression size is 
\begin{align*}
    n(d+d\log\lr{1/\alpha})
    &\leq
    2nd\log\lr{1/\alpha}.
\end{align*}
Plugging in $n$ and $d$, we conclude
\[
\begin{dcases}
	\Tilde{\cO}\lr{\frac{1}{\beta^2}\fat\lr{\cF,c\alpha/p}\fat^*\lr{\cF,c\alpha/p}
 }, p\in [1,\infty)\\
    \Tilde{\cO}\lr{\frac{1}{\beta^2}\fat\lr{\cF,c\alpha}\fat^*\lr{\cF,c\alpha}
    }, p=\infty.
 \end{dcases}
\]

\section{Agnostic Compression for Linear Regression}\label{sec:linear-regression}
In this section, our focus is on $\ell_p$ linear regression in $\R^d$. We begin by improving upon the construction of an approximate sample compression scheme for general classes, incorporating the structure of linear functions.
Next, we demonstrate the feasibility of constructing an exact compression for $p\in \lrset{1,\infty}$ with a size linear in $d$. In sharp contrast, we exhibit that this holds only for $p\in \lrset{1,\infty}$. 
We prove an impossibility result of achieving a bounded-size exact compression scheme for $p\in (1,\infty)$. 
 
We use the following notation. Vectors $\vec v\in\R^d$ are denoted by boldface,
and their $j$th coordinate is indicated by $\vec v(j)$.
(Thus, $\vec v_i(j)$ indicates the $j$th coordinate of the $i$th vector
in a sequence.)
\subsection{Approximate Compression for $p\in \lrbra{1,\infty}$}\label{subsec:approximate-compression-linear}
In this subsection, our instance space is $\X=[0,1]^d$,
label space is $\Y=[0,1]$,
and hypothesis class is bounded homogeneous linear functions $\F\subseteq\Y^\X$, consisting of all
$f_{\vec w}:\X\to\Y$ given by
$f_{\vec w}(\vec x)=\inner{\vec w,\vec x}$,
indexed by $\vec w \in \R^d$, where $\norm{\vec w}_2\leq 1$.

In \Cref{sec:compression-general-classes} we proved an approximate compression for general function classes with $\ell_p$ losses of size
$\cO\lr{\fat_{c\alpha/p}\cdot \fat^*_{c\alpha/p}\cdot\polylog\lr{\fat_{c\alpha/p},\fat^*_{c\alpha/p},p,1/\alpha}}.$
We have an immediate corollary for linear functions. Let $\pseudo\lr{\cF}$ be the pseudo-dimension of a function class $\F$ \cite{pollard1990empirical,haussler1992decision}, that can be defined as
$\pseudo(\cF) = \underset{\gamma \to\ 0}{\lim}\fat_\gamma(\cF)$.
The fat-shattering dimension (at any scale) is upper bounded by the pseudo-dimension. 
Moreover, the vector space dimension is of the same order as the pseudo-dimension \citep{anthony1999neural}, and the dimension of the dual vector space is equal to the one of the primal space. This implies the following.
\begin{corollary}\label{cor:approx-comprssion-linear}
\cref{alg:compression} is a 
    sample compression scheme of size $\cO\lr{d^2\cdot\polylog\lr{d,p,\frac{1}{\alpha}}}$ for bounded linear regression in dimension $d$ with the $\ell_p$ loss, for $p\in [1,\infty]$.
\end{corollary}
In this section, we improve upon the result by using a dedicated algorithm for linear functions. We prove the following:
%
\begin{theorem}[Approximate compression for agnostic linear regression]\label{thm:lin-regression-approx-comprssion} 
Let $\cF = \lrset{\vec x\mapsto \inner{\vec w, \vec x}: \vec w \in \R^d, \norm{\vec w}_2\leq 1 }$, $S=\lrset{\lr{\vec x_i,y_i}:  \norm{\vec x_i}_2 \leq 1, \forall i\in [m]}\subseteq \cX\times [0,1]$, and an approximation parameter $\alpha\in (0,1)$.
\Cref{alg:lin-regression-lp} is an $\alpha$-approximate sample compression scheme for the $\ell_p$ loss of size
\[ 
\begin{dcases}
    \cO\lr{d\cdot\log\lr{\frac{p}{\alpha}}}, & p\in [1,\infty) \\
     \cO\lr{d\cdot\log\lr{\frac{1}{\alpha}}}, & p=\infty. 
\end{dcases}   
\]
\end{theorem}
\begin{algorithm}[H]
    \caption{Approximate Agnostic Compression for $\ell_p$ Linear Regression, $p \in[1,\infty]$}\label{alg:lin-regression-lp}
    \textbf{Input}:
    $\cF = \lrset{\vec x\mapsto \inner{\vec w, \vec x}: \vec w \in \R^d, \norm{\vec w}_2\leq 1 }$, $S=\lrset{\lr{\vec x_i,y_i}:  \norm{\vec x_i}_2 \leq 1, \forall i\in [m]}\subseteq \cX\times [0,1]$.\\
    \textbf{Parameters}: Approximation parameter $\alpha\in [0,1]$.
    \begin{enumerate}[leftmargin=0.43cm]
            \item[$\triangleright$]\textcolor{DarkBlue}{Find an optimal regressor for $S$}
            \item $f^\star\leftarrow \argmin_{f\in\cF}L_p(f,S)$
            \item[$\triangleright$]\textcolor{DarkBlue}{Define a discretized dataset where the new labels are discretized to a resolution of $\alpha$}
            \item Define $S_\alpha= A \cup B$, where
            \begin{align*}
            A&=\Bigl\{\lr{\vec x_i,j\alpha}: i\in [m], j\in \lrset{-1/\alpha,\ldots,-1,0,1,\ldots,1/\alpha}\Bigl\}
            \\
            B &= \Bigl\{\lr{\vec x_i,j(1+\alpha)}: i\in [m], j\in \lrset{-1,+1}\Bigl\}
            \end{align*}
            %
            \item[$\triangleright$]\textcolor{DarkBlue}{Label by $\pm 1$ the discretized dataset with $f^\star$ }
            \item
            Define 
            $$S_\alpha(f^\star)=\{\lr{\lr{\vec x_i,\tilde{y}},z}:\text{ for any }(\vec x_i,\tilde{y})\in S_\alpha:$$ 
            $$z=+1 \text{ if } f^\star(\vec x_i)-\tilde{y} \leq 0 \text{, otherwise } z=-1\}$$
        \end{enumerate}
    \textbf{Compression:} Run \texttt{SVM} for realizable binary classification on $S_\alpha(f^\star)$ and return a set of \emph{support vectors}. \\
    \textbf{Reconstruction:} Run \texttt{SVM} and the compression set.
\end{algorithm}

\begin{proof} 
Let $\cF$ be the set of homogeneous linear predictors bounded by $1$,
$\cF = \lrset{\vec x\mapsto \inner{\vec w, \vec x}: \vec w \in \R^d, \norm{\vec w}_2\leq 1 }$, and data a set $S=\lrset{\lr{\vec x_i,y_i}:  \norm{\vec x_i}_2 \leq 1, \forall i\in [m]}\subseteq \cX \times [0,1]$.

\emph{Approximate compression correctness.}
The algorithmic idea is as follows. We first compute in Step 1 an optimal linear regressor $f^\star\in \cF$ for the $\ell_p$ loss. In step 2, we create a discretized dataset $S_\alpha$ of size $m(2/\alpha+3)$, where for each example $\vec x_i$ we create $(2/\alpha+3)$ real-valued labels $\lrset{-1-\alpha,-1,\ldots,-2\alpha,-\alpha,0,\alpha,2\alpha,\ldots,1,1+\alpha}$. Then in step 3, we use the regressor $f^\star$
for classifying the dataset $S_\alpha$. That is, for any $(\vec x_i, \tilde{y})\in S_\alpha$, we have $(\lr{\vec x_i,\tilde{y}}, +1)$ whenever $f^\star(\vec x_i)-\tilde{y} \leq 0$, and $(\lr{\vec x_i,\tilde{y}}, -1)$ otherwise. We denote this dataset by $S_\alpha(f^\star)$. Note that for each $\vec x_i$ we created a grid of binary labels of resolution $\alpha$ in the range $[-1-\alpha,1+\alpha]$, and since $\lrabs{f^\star(\vec x_i)}\leq 1$, for each vector $\vec x_i$ 
there exists  $\Tilde{y}_1,\Tilde{y}_2$ such that $(\vec x_i,\Tilde{y}_1)$, $(\vec x_i,\Tilde{y}_2)\in  S_\alpha(f^\star)$ have different labels.
To obtain compression, we execute \texttt{Support Vector Machine(SVM)} for realizable classification on $S_\alpha(f^\star)$.
Note that the classification problem is in $\R^{d+1}$ and the original regression problem is in $\R^d$. 
Applying Caratheodory's theorem allows us to express its output as a linear combination of $d+2$ support vectors (along with their labels). The set of returned support vectors constitutes the compression set. For reconstruction, we utilize \texttt{SVM} on these support vectors. The hyperplane returned by \texttt{SVM} can be re-interpreted as a function from $\R^d$ to $\R$ that pointwise approximates $f^\star$ on all $\vec x_i$ in $S$.

We proceed to prove the correctness. Denote the output of the compression scheme by $f_{\svm}=\rho\lr{\kappa\lr{S}}=(\vec w_{\svm},b_{\svm})$, which a affine linear function in $\R^{d+1}$. This function can be re-interpreted as an affine linear function $\hat{f}:\R^d \rightarrow \R$, for any $\vec x\in \R^{d}$ we compute $y\in\R$ by solving $\inner{\vec w_{\svm}, (\vec x, y)}+b_{\svm}=0$, 
$$\hat{f}(\vec x)=y=\frac{\inner{\vec w_{\svm}^d, \vec x}+b_{\svm}}{\vec w_{\svm}(d+1)},$$
where $\vec w_{\svm}^d= \lr{\vec w_{\svm}(1),\ldots,\vec w_{\svm}(d)}$.
It holds that $\vec w_{\svm}(d+1)\neq 0$, since for any $\vec x_i$ there exists  $\Tilde{y}_1,\Tilde{y}_2$ such that $(\vec x_i,\Tilde{y}_1)$, $(\vec x_i,\Tilde{y}_2)\in  S_\alpha(f^\star)$ have different labels. If $\vec w_{\svm}(d+1)=0$ it means that the \texttt{SVM} hyperplane cannot distinguish between these two points, and thus, it makes a mistake on a realizable dataset, which is a contradiction. 
Since the output of \texttt{SVM} is a valid compression scheme for realizable binary classification, $\hat{f}$ should classify correctly all points in  
$S_\alpha(f^\star)$.
It follows that for any $\vec x_i$ in $S$,
$$\lrabs{f^\star(\vec x_i)-\hat{f}\lr{\vec x_i}}\leq \alpha,$$ 
due to the two adjacent grid points with resolution $\alpha$ lying above and below both the hyperplane of $f^\star$ and the $\hat{f}$ hyperplane.
%
%
Therefore, for any $(\vec x_i, y_i)\in S$
\begin{align*}
    \lrabs{|f^\star(\vec x_i)-y_i| -|\hat{f}(\vec x_i)-y_i|}
    &\overset{(i)}{\leq}
    \lrabs{f^\star(\vec x_i)-y_i -\hat{f}(\vec x_i)+y_i} 
    \\
    &=
    \lrabs{f^\star(\vec x_i)-\hat{f}(\vec x_i)}
    \\
    &\leq
    \alpha,
\end{align*}
where (i) follows from the triangle inequality, 
and so $\hat{f}$ is an $\alpha$-approximate sample compression scheme for the $\ell_1$ and $\ell_\infty$ losses. For $p\in (1,\infty)$, using Lipschitzness of the $\ell_p$ loss, we have 
%
\begin{align*}
\lrabs{|f^\star(\vec x_i)-y_i|^p -|\hat{f}(\vec x_i)-y_i|^p}
\end{align*}
\begin{align*}
    &\leq 
    \lrabs{p\lr{|f^\star(\vec x_i)-y_i| -|\hat{f}(\vec x_i)-y_i|}}
    \\
    &=
    p\lrabs{|f^\star(\vec x_i)-y_i| -|\hat{f}(\vec x_i)-y_i|}
    \\
    &\leq
    p\alpha.
\end{align*}

By rescaling $p\alpha$ to $\alpha$, we have an $\alpha$-approximate compression scheme for the $\ell_p$ loss.

\emph{Approximate compression size.} The \texttt{SVM} running on $S_\alpha(f^\star)$ returns a set of support vectors of size at most $d+2$, since the input is in dimension $d+1$. The $\vec x$ vectors are part of the original sample $S$. We need to keep the grid point labels of the support vectors as well, each one of them requires $\log(1/\alpha)$ bits, and each classification $\pm 1$ costs an extra bit. We get a compression of size
$
    d+2+(d+2)\log\lr{1/\alpha} +d+2
    =
    \cO\lr{d\log\lr{1/\alpha}}.
$
%
\end{proof} 
\subsection{Exact Compression for $p\in\lrset{1,\infty}$}\label{subsec:exact-compression-l1-linfty}
In this section, 
we show that
agnostic linear regression in $\R^d$ admits
an \emph{exact} compression scheme of size $d+1$ under $\ell_1$
and $d+2$ under $\ell_\infty$.
Our instance space is $\X=\R^d$,
label space is $\Y=\R$,
and hypothesis class is $\F\subseteq\Y^\X$, consisting of all
$f_{\vec w,b}:\X\to\Y$ given by
$f_{\vec w,b}(\vec x)=\inner{\vec w,\vec x}+b$,
indexed by $\vec w\in\R^d,b\in\R$. 
Note that we allow unbounded norms for the linear functions and the data can be unbounded as well, as opposed the the results in \Cref{subsec:approximate-compression-linear}.
%
\begin{theorem}
  \label{thm:ell1} 
  There exists an efficiently computable (see the linear program in \cref{eq:reg-lp1-general-d}) exact compression scheme
  for agnostic $\ell_1$ linear regression of size $d+1$. 
\end{theorem}
The optimization technique based on minimizing the sum of absolute deviations is known as Least Absolute Deviations (LAD) and was introduced by Boscovich in 1757 (see, for example, \citet{dodge2008least}). We derive a compression scheme from this method. 
\begin{proof}
We start with $d=0$. The sample then consists of $(y_1,\ldots,y_m)$
[formally: pairs $(x_i,y_i)$, where $x_i\equiv0$],
and $\F=\R$ [formally, all functions $h:0\mapsto\R$].
We define $f_S$ to be the median of $(y_1,\ldots,y_m)$,
which for odd $m$ is defined uniquely and for even $m$ can be taken
arbitrarily as the smaller of the two midpoints.
It is well-known that such a choice minimizes the empirical
$\ell_1$ risk, and it clearly constitutes a compression scheme of size $1$.

The case $d=1$ will require more work.
The sample consists of $(x_i,y_i)_{i\in[m]}$,
where $x_i,y_i\in\R$,
and $\F=\set{
  \R\ni x\mapsto wx+b: a,b\in\R
}$.
Let $(w^\star,b^\star)$ be a (possibly non-unique) minimizer of
\beqn
\label{eq:1dim-prob}
L(w,b):=\sum_{i\in[m]}|(wx_i+b)-y_i|,
\eeqn
achieving the value
$L^\star$. We claim that we can always find two indices
$\bari,\barj\in[m]$ such that the line determined by
$(x_\bari,y_\bari)$ and $(x_\barj,y_\barj)$ also achieves the optimal
empirical risk $L^\star$. 
More precisely,
the line $(\hat w,\hat b)$ induced by
$((x_\bari,y_\bari),(x_\barj,y_\barj))$
via\footnote{We ignore the degenerate possibility of vertical lines,
which reduces to the $0$-dimensional case.}
$\hat w=(y_\barj-y_\bari)/(x_\barj-x_\bari)$
and
$\hat b=y_\bari-\hat w x_\bari$,
verifies $L(\hat w,\hat b)=L^\star$.

To prove this claim, we begin by recasting (\ref{eq:1dim-prob})
as a linear program.
%
\begin{align}
  \label{eq:reg-lp1}
  \min_{(\eps_1,\ldots,\eps_m,w,b)\in\R^{m+2}}
  &\quad \sum_{i=1}^m \eps_i \quad
  \mbox{s.t.}& \\
\forall i\in[m] &\quad \eps_i\ge0 \nonumber\\
\forall i\in[m] &\quad wx_i+b-y_i \le \eps_i \nonumber\\
\forall i\in[m] &\quad -wx_i-b+y_i \le \eps_i \nonumber
.
\end{align}
%
We observe that the linear program in (\ref{eq:reg-lp1})
is feasible with a finite solution
(and actually, the constraints $\eps_i\ge0$ are redundant).
Furthermore, any optimal
value is achievable at one of the extreme points of the constraint-set
polytope $\mathcal{P}\subset\R^{m+2}$.
Next, we claim that
the
extreme points of
the polytope
$\mathcal{P}$
are all of the form
$v\in\mathcal{P}
$
with
two
(or more)
of the $\eps_i$s equal to $0$.
This
suffices to prove our main claim,
since $\eps_i=0$ in
$v\in
\mathcal{P}
$ iff the
$(w,b)$ induced by $v$ verifies $wx_i+b=y_i$;
in other words, the line induced by $(w,b)$
contains the point $(x_i,y_i)$.
If a line contains two data points,
it is uniquely determined by them:
these constitute a compression set of size $2$.
(See illustration in Figure~\ref{fig:plots}.)

\newcommand\TikCircle[1][2.5]{\tikz[baseline=-#1]{\draw[thin,blue,fill=blue](0,0)circle[radius=#1pt];}}

\begin{figure}
 \includegraphics[width=0.5\textwidth]{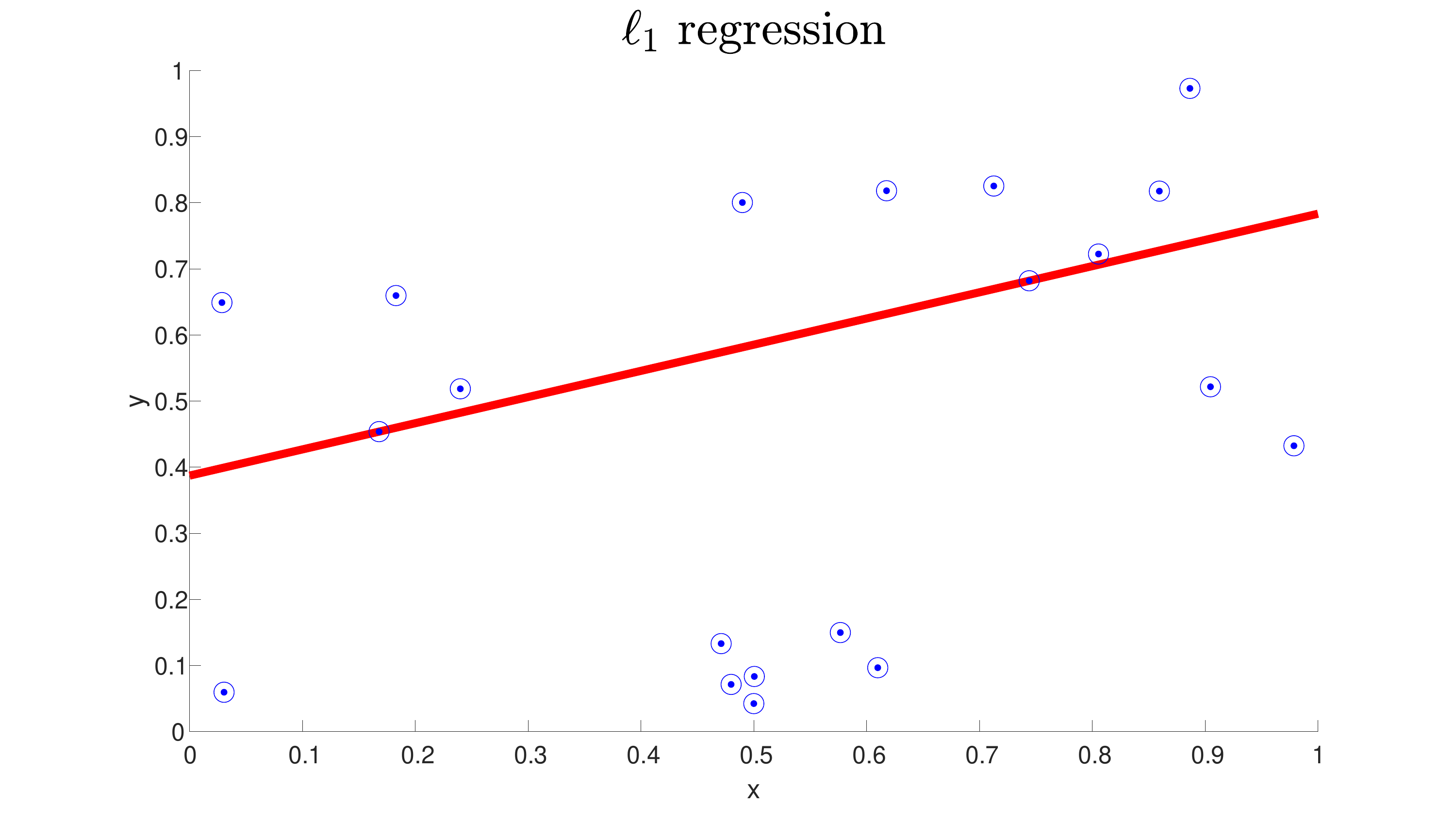}
  \includegraphics[width=0.5\textwidth]{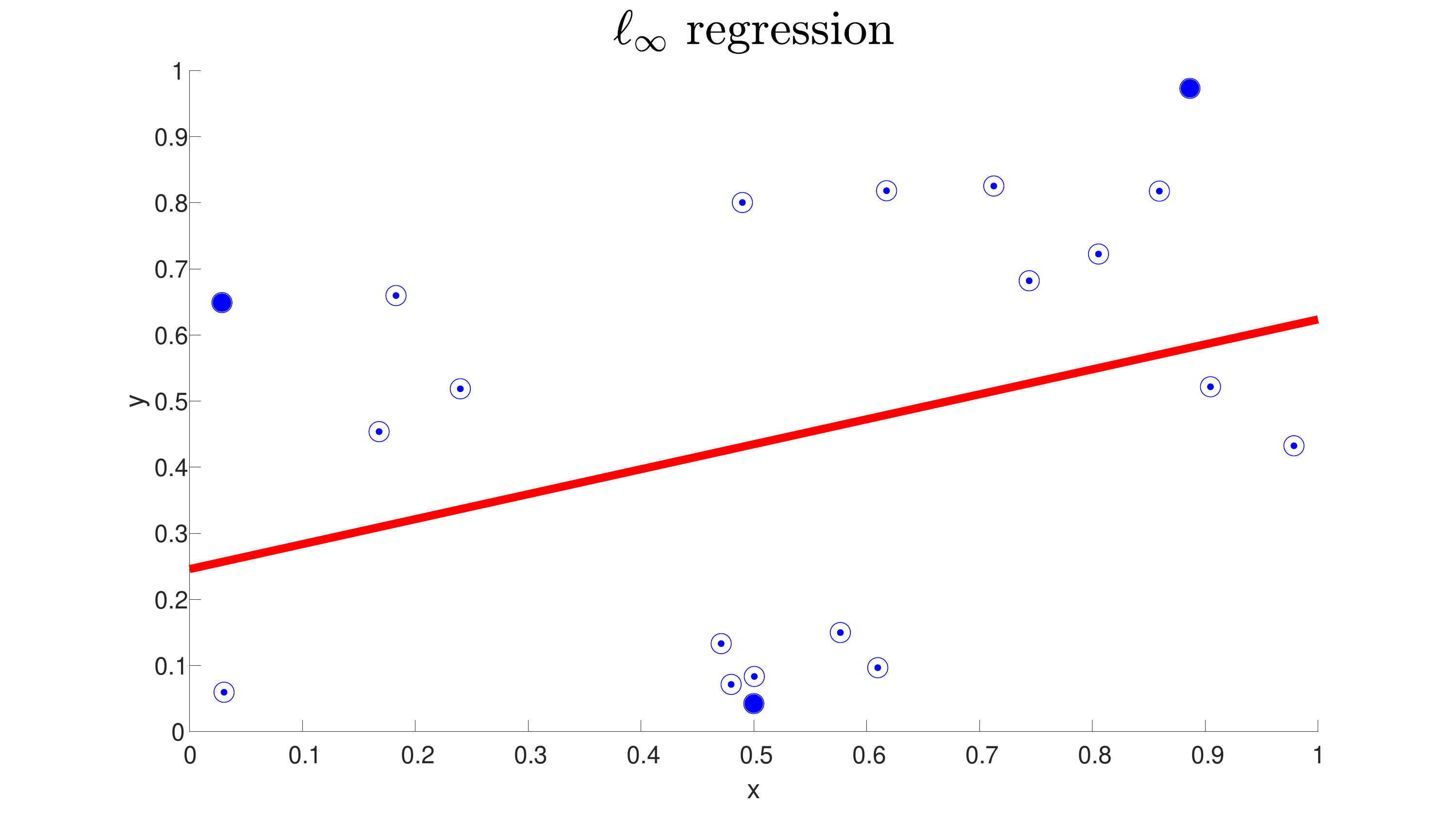}    \\
  \caption{
    A sample $S$ of $m=20$ points $(x_i,y_i)$
    was drawn iid uniformly from $[0,1]^2$.
    On this sample, $\ell_1$ regression was performed by solving
    the LP in (\ref{eq:reg-lp1}), shown on the left,
    and $\ell_\infty$ regression was performed by solving the LP in
    (\ref{eq:reg-lpinf}), on the right.
    In each case, the regressor provided by the LP solver
    is indicated by the thick (red) line. Notice that for $\ell_1$, the line
    contains exactly $2$ datapoints. For $\ell_\infty$, the regressor contains
    no datapoints; rather, the $d+2=3$ ``support vectors'' are indicated by~$
\TikCircle
    $.
}
  \label{fig:plots}
\end{figure}  

Now we prove our claimed
property
of the extreme points.
First, we claim that any
extreme point of $\mathcal{P}$
must have at least one
$\eps_i$ equal to $0$.
Indeed, let $(w,b)$ define a line.
Define
$$b^+ := \min\set{\tilde b\in[b,\infty): \exists i\in[m], wx_i +
    \tilde b
    = y_i)}$$
and analogously,
$$b^- := \max\set{\tilde b\in(-\infty,b] : \exists i\in[m], wx_i +
\tilde b = y_i)}.$$
In words,
$(w, b^+)$ is the line obtained by
increasing $b$ to a maximum value of $b^+$,
where the line $(w,b^+)$ touches a datapoint,
and likewise,
$(w, b^-)$ is the line obtained by
decreasing $b$ to a minimum value of $b^-$,
where the line $(w,b^-)$ touches a datapoint.

Define by $S_{a,b}^+ := \set{i : \abs{wx_i+b < y_i}}$
the points above the line defined by $(w,b)$
and $S_{a,b}^- := \set{i : \abs{wx_i+b > y_i}}$
the points below the line defined by $(w,b)$.
For a line $(w,b)$ which does not contain a data point
we can rewrite the sample loss as
\beq
L(w, b) &=&
\sum_{i\in S_{a,b}^+} (y_i - (wx_i+b)) + \sum_{i\in S_{a,b}^-} ((wx_i+b) - y_i)
\\ &=&
\paren{\sum_{i\in S_{a,b}^-} x_i - \sum_{i\in S_{a,b}^+} x_i}a
+
\paren{|S_{a,b}^-| - |S_{a,b}^+|}b
+
\paren{\sum_{i\in S_{a,b}^+} y_i - \sum_{i\in S_{a,b}^-} y_i}
\\ &=:&
\lambda a+\mu b+\nu
.
\eeq

Since
for fixed $a$ and $b\in[b^-,b^+]$,
the quantities
$S_{a,b}^-, S_{a,b}^+$ are constant,
it follows that
the function $L(w,\cdot)$ is
affine in $b$, and hence minimized at
$b^{\pm}\in\set{b^-,b^+}$.
Thus,
there is no loss of generality in taking $b^\star=b^\pm$,
which implies 
that the optimal solution's line
$(w^\star,b^\star)$
contains a data point $(x_\bari,y_\bari)$.
If the line $(w^\star,b^\pm)$ contains other data points
then we are done, so assume
to the contrary
that $\eps_\bari$ is the only
$\eps_i$ that vanishes in the corresponding solution
$v^\star\in\mathcal{P}$.

Let $\mathcal{P}_\bari\subset\mathcal{P}$
consist of all $v$ for which $\eps_\bari=0$,
corresponding to all feasible solutions whose line contains the
data point $(x_\bari,y_\bari)$.
Let us say that two lines $(w_1,b_1),(w_2,b_2) $ are {\em equivalent} if they induce
the same partition on the data points, in the sense of linear
separation in the plane.
The formal condition is
$S_{w_1,b_1}^- = S_{w_1,b_1}^-$,
which is equivalent to
$S_{w_1,b_1}^+ = S_{w_1,b_1}^+$.

Define
$
\mathcal{P}_\bari^\star
\subset\mathcal{P}_\bari$
to consist of those feasible solutions whose line is equivalent
to $(w^\star,b^\pm)$.
Denote by
$w^+ := \max\set{a : (\varepsilon_1,..,\varepsilon_m,w,b)\in \mathcal{P}_\bari^\star}$
and define
$v^+$ to be a feasible solution in $\mathcal{P}_\bari^\star$
with slope $w^+$,
and analogously,
$w^- := \min\set{w : (\varepsilon_1,..,\varepsilon_m,w,b)\in \mathcal{P}_\bari^\star} $
and $v^-\in\mathcal{P}_\bari^\star$
with slope $w^-$.
Geometrically this corresponds
to
rotating
the line $(w^\star,b^\star)$
about the point $(x_\bari,y_\bari)$ until it encounters
a data point above
and below.

Writing, as above, the sample loss in the form
$L(w,b)$, we see that $L(\cdot,b^\pm)$
is affine in $a$ over the range $w\in[w^-,w^+]$
and hence is minimized at one of the endpoints.
This furnishes another datapoint $(x_\barj,y_\barj)$
verifying $\hat w x_\barj+\hat b=y_\barj$
for $L(\hat w,\hat b)=L^\star$,
and hence proves compressibility
into two points for $d=1$.

Generalizing to $d>1$ is quite straightforward.
We define
\beq
L(\vec w,b)=\sum_{i\in[m]}|(\inner{\vec w,\vec x_i}+b)-y_i|
\eeq
and express it as a linear program
analogous to 
(\ref{eq:reg-lp1}),
\begin{mdframed}[userdefinedwidth=8cm,align=center]
\textbf{Linear programming for $\ell_1$ regression}:
\begin{align}
  \label{eq:reg-lp1-general-d}
  \min_{(\eps_1,\ldots,\eps_m,\vec w,b)\in\R^{m+d+1}}
  &\quad \sum_{i=1}^m \eps_i \quad
  \mbox{s.t.}& \\
\forall i\in[m] &\quad \eps_i\ge0 \nonumber\\
\forall i\in[m] &\quad \inner{\vec w,\vec x_i}+b-y_i \le \eps_i \nonumber\\
\forall i\in[m] &\quad -\inner{\vec w,\vec x_i}-b+y_i \le \eps_i \nonumber
.
\end{align}
\end{mdframed}
Given an optimal solution $(\vec w^\star,b^\star)$,
we argue exactly as above that $b^\star$
may be chosen so that the optimal regressor contains
some datapoint --- say, $(\vec x_1,y_1)$.
Holding $b^\star$ and $\vec w(j)$, $j\neq 1$
fixed, we argue, as above, that $\vec w(1)$
may be chosen so that the optimal regressor
contains another datapoint (say, $(\vec x_2,y_2)$).
Proceeding in this fashion, we inductively
argue that the optimal regressor may be
chosen to contain some $d+1$ datapoints,
which provides the requisite compression scheme.
\end{proof}
Similarly, we can obtain a compression scheme for $\ell_\infty$ loss via linear programming.
\begin{theorem}
  \label{thm:ellinfty}
  There exists an efficiently computable (see the linear program in \cref{eq:reg-lpinf}) exact compression scheme
  for agnostic $\ell_\infty$ linear regression of size $d+2$.
\end{theorem}
\begin{proof}
Given $m$ labeled points in $\R^d \times \R$,
$S = \lrset{(\vec x_i,y_i):i
\in [m]}$
and any $\vec w \in \R^d$, $b \in \R$ define the empirical risk
\beq
L(\vec w,b) &:=&
\max
\set{|\inner{\vec w,\vec x_i} + b - y_i|:i\in[m]}
.
\eeq
We cast the risk minimization problem
as a linear program.
\begin{mdframed}[userdefinedwidth=8cm,align=center]
\textbf{Linear programming for $\ell_\infty$ regression}:
\begin{align}
  \label{eq:reg-lpinf}
  \min_{(\eps, \vec w, b)\in\R^{d+2}} : & \quad \eps \\
  s.t.\quad  \forall i : & \quad \eps - \inner{\vec w,\vec x_i} - b + y_i \geq 0 \nonumber\\
  & \quad \eps + \inner{\vec w,\vec x_i} + b - y_i \geq 0 \nonumber
  .
\end{align}
\end{mdframed}


(As before, the constraint $\eps\ge0$ is implicit in the other constraints.)
Introducing the Lagrange multipliers $\lambda_i,\mu_i\ge 0$, $i\in[m]$,
we cast
the optimization problem
in
the
form
of a
Lagrangian:
\beq
  \lagrangian(\eps, \vec w, b, \mu_1\ldots,\mu_m,\lambda_1\ldots,\lambda_m)
  &=&
  \eps
     - \sum_{i=1}^m \lambda_i \left(\eps - \inner{\vec w,\vec x_i} - b + y_i \right)
     - \sum_{i=1}^m \mu_i \left(\eps + \inner{\vec w,\vec x_i} + b - y_i \right)
     .
\eeq

The KKT conditions imply, in particular, that
\begin{align*}
  \forall i: \quad & \lambda_i(\eps - \inner{\vec w,\vec x_i} - b + y_i) = 0 \\
  & \mu_i(\eps + \inner{\vec w,\vec x_i} + b - y_i) = 0
  .
\end{align*}

Geometrically, this means that either the constraints corresponding to
the $i$th datapoint are inactive --- in which case, omitting the datapoint
does not affect the solution --- or otherwise,
the $i$th datapoint induces the active constraint
\beqn
\inner{\vec w,\vec x_i} + b - y_i
=
\eps
.
\label{eq:svm}
\eeqn
On analogy with SVM, let us refer to the datapoints satisfying
(\ref{eq:svm})
as the {\em support vectors}; clearly,
the remaining sample points may be
discarded without affecting the solution.
Solutions to (\ref{eq:reg-lpinf}) lie in $\R^{d+2}$
and hence $d+2$ linearly independent datapoints
suffice to uniquely pin down an optimal $(\eps,\vec w,b)$
via the equations (\ref{eq:svm}).

\end{proof}
\subsection{Exact Constant Size Compression Is Impossible for $p\in (1,\infty)$}
\label{subsec:impossible-exact-compression}
We proceed to show that it is impossible to have an \emph{exact} compression scheme of constant size (independent of the sample size) for $p\in (1,\infty)$, generalizing the result for the $\ell_2$ loss by \citet[Theorem 4.1]{david2016supervised}.
\begin{theorem}[\citet{david2016supervised}]
  \label{thm:david-ell2}
  There is no exact agnostic sample compression scheme for zero-dimensional linear
regression with size $k(m)\le m/2$.
\end{theorem}
%
%
\begin{theorem}
  \label{thm:our-lb}
  There is no exact agnostic sample compression scheme for zero-dimensional
  linear regression
under $\ell_p$ loss, $1<p<\infty$,
with size $k(m) < \log(m)$.
\end{theorem}
\begin{proof}
  Consider a sample
  $(y_1,\ldots,y_m)\in\set{0,1}^m$.
  Partition the indices $i\in[m]$
  into
  $S_0 := \{i\in[m]: y_i = 0\}$
  and
  $S_1 := \{i\in[m]: y_i = 1\}$.
  The empirical risk minimizer is given by
  \beq
  \hat r:=\argmin_{s\in\R}\sum_{i=1}^m |y_i - s|^p.
  \eeq

  To obtain an explicit expression for $\hat r$, define
  \beq
  F(s) = \sum_{i=1}^m |y_i - s|^p
  = |S_1|(1-s)^p + |S_0|s^p
  =:N_1(1-s)^p + N_0s^p.
  \eeq
  

We then compute
\beq
F'(s)=pN_0s^{p-1}-pN_1(1-s)^{p-1}
\eeq
and find that $F'(s)=0$ occurs at
\beq
\hat s = \frac{\mu^{1/(p-1)}}{1+\mu^{1/(p-1})},
\eeq
where $\mu=N_1/N_0$.
A straightforward analysis of the second derivative
shows that $\hat{s}=\hat r$ is indeed the unique minimizer of $F$.

Thus, given a sample of size $m$, the unique minimizer $\hat r$ 
is uniquely
determined by $N_0$ --- which can take on any of integer $m+1$ values
between $0$ and $m$.
On
the other hand,
every
output of a $k$-selection function $\kappa$
outputs a multiset $\hat S\subseteq S$ of size $k'$
and a binary string of length $k''=k-k'$.
Thus, the total number of values representable by a $k$-selection scheme
is at most
\beq
\sum_{k'=0}^k {k'}2^{k-k'}<2^{k+1}-k,
\eeq
which, for $k<\log m$, is less than $m$.
\end{proof}

\begin{remark}
  \label{rem:lb-Omega(m)}
  A more refined analysis, along the lines of
  \citet[Theorem 4.1]{david2016supervised}, should
  yield a lower bound of $k=\Omega(m)$. A technical complication
  is that unlike the $p=2$ case, whose empirical risk minimizer
  has a simple explicit form, the general $\ell_p$ loss
  does not admit a closed-form solution and uniqueness must be argued
  from general convexity principles. 
\end{remark}
%
\section{Open Problems}\label{sec:open-problems}
The positive result for $\ell_1$ loss may also lead us to wonder 
how general of a result might be possible.  In particular, noting that 
the pseudo-dimension \citep{pollard84,MR1089429,anthony1999neural} of linear functions in $\R^d$ 
is precisely $d+1$ \citep{anthony1999neural}, 
there is an intriguing possibility for the following generalization.
For any class $\mathcal{F}$ of real-valued functions, denote by $\pseudo\lr{\cF}$ 
the pseudo-dimension of $\mathcal{F}$.

\paragraph{Open Problem: Compressing to pseudo-dimension Number of Points.}
Under the $\ell_1$ loss, 
does every class $\mathcal{F}$ of real-valued functions 
admit an \emph{exact} agnostic compression scheme of size $\pseudo\lr{\cF}$?

It is also interesting, and perhaps more approachable as an initial aim, 
to ask whether there is an agnostic compression scheme of size 
at most \emph{proportional to} $\pseudo\lr{\cF}$.  Even falling short of this, 
one can ask the more-basic question of whether classes with $\pseudo\lr{\cF} < \infty$
always have \emph{bounded} agnostic compression schemes (i.e., independent of sample size $m$), 
and more specifically whether the bound is expressible purely as a function of $\pseudo\lr{\cF}$
(\citet{moran2016sample} have shown this is always possible in the realizable classification setting).

These questions are directly related to (and inspired by) the well-known 
long-standing conjecture of 
\citet{DBLP:journals/ml/FloydW95,warmuth2003compressing}, which 
asks whether, for realizable-case binary classification, 
there is always a compression scheme of size at most linear in the VC dimension
of the concept class.
Indeed, it is clear that a positive solution of 
our open problem above would imply a positive solution to the 
original sample compression conjecture, since in the realizable case 
with a function class $\mathcal{F}$ of $\{0,1\}$-valued functions, 
the minimal empirical $\ell_1$ loss on the data is zero, and any function obtaining zero 
empirical $\ell_1$ loss on a data set labeled with $\{0,1\}$ values 
must be $\{0,1\}$-valued on that data set, and thus can be thought of as a 
sample-consistent classifier.\footnote{To make such a function actually binary-valued everywhere, 
it suffices to threshold at $1/2$.}  Noting that, for $\mathcal{F}$ containing 
$\{0,1\}$-valued functions, $\pseudo\lr{\cF}$ is equal the VC dimension, the implication is clear. 

The converse of this direct relation is not necessarily true.
Specifically, for a set $\mathcal{F}$ of real-valued functions, 
consider the set $\mathcal{H}$ of subgraph sets: $h_{f}(x,y) = \mathbb{I}[ y \leq f(x) ]$, $f \in \mathcal{F}$.
In particular, note that the VC dimension of $\mathcal{H}$ is precisely $\pseudo\lr{\cF}$.
It is \emph{not} true that any realizable classification compression scheme for $\mathcal{H}$ 
is also an agnostic compression scheme for $\mathcal{F}$ under $\ell_1$ loss.
Nevertheless, this reduction-to-classification approach seems intuitively appealing, 
and it might possibly be the case that there is some way to \emph{modify} certain types 
of compression schemes for $\mathcal{H}$ to convert them into agnostic compression schemes for $\mathcal{F}$.
Following up on this line of investigation seems the natural next step 
toward resolving the above general open question.

Similarly, we ask the analogous question for the $\ell_2$ loss and approximate sample compression schemes.

\paragraph{Open Problem: Compressing to fat-shattering Number of Points.}
Let $c>0$ be an absolute constant.
Under the $\ell_2$ loss, 
does every class $\mathcal{F}$ of real-valued functions 
admit an $\alpha$-\emph{approximate} agnostic compression scheme of size $\fat\lr{\cF,c\alpha}\cdot\polylog\lr{c/\alpha}$?

\newpage
\bibliography{refs}
\appendix

\section{Auxiliary Proofs for \Cref{sec:compression-general-classes}}\label{app:auxiliary-compression-general-classes}

Our proof relies on several auxiliary results. 
\paragraph{Existence of approximate weak learners.}
We start with a result about generalization from interpolation. \citet{anthony2000function} established such a result for interpolation models (\citet[Section 21.4]{anthony1999neural}), where the cut-off parameter $\eta>0$ is fixed. The following results extend to cut-offs that may differ for different points. 
A similar result appeared in \citet{attias2023adversarially} in the context of adversarially robust learning.
%
%
\begin{theorem}
[Generalization from approximate interpolation with changing cutoffs]
\label{thm:generalization-interpolation-cutoffs}
Let $\cF\subseteq [0,1]^\cX$ be a function class with a finite fat-shattering dimension (at any scale). For any $\alpha,\epsilon,\delta\in(0,1)$, any function $\psi:\cX\times \cY \rightarrow [0,1]$, any distribution $P$ over $\cX\times\cY$, for a random sample $S\sim P^m$, if 
$$
m=\cO\lr{\frac{1}{\epsilon}\lr{\fat\lr{\cF,\alpha/8}\log^2\lr{\frac{\fat\lr{\cF,\alpha/8}}{\alpha\epsilon}}+\log\frac{1}{\delta}}},
$$ 
then with probability at least $1-\delta$ over $S$, for any ${f}\in\cF$ satisfying $\lrabs{{f}(x)-y}\leq \psi(x,y)+\alpha$, $\forall{(x,y)\in S}$,
it holds that $\hP_{(x,y)\sim P}\lrset{(x,y): \lrabs{f(x)-y} \leq \psi(x,y)+2\alpha}\geq 1- \epsilon$.
\end{theorem}

\begin{theorem}[\textbf{Generalization from approximate interpolation}]\citep[Theorems 21.13 and 21.14]{anthony1999neural}\label{thm:generalization-interpolation}
%
%
Let $\cF\subseteq [0,1]^\cX$ be a function class with a finite fat-shattering dimension (at any scale). For any $\eta,\alpha,\epsilon,\delta\in(0,1)$, any distribution $\cD$ over $\cX$, any function $t:\cX \rightarrow [0,1]$, for a random sample $S\sim \cD^m$, if 
$$
m(\eta, \alpha,\epsilon,\delta)=\cO\lr{\frac{1}{\epsilon}\lr{\fat\lr{\cF,\alpha/8}\log^2\lr{\frac{\fat\lr{\cF,\alpha/8}}{\alpha\epsilon}}+\log\frac{1}{\delta}}},
$$ 
then with probability at least $1-\delta$ over $S$, for any ${f}\in\cF$ satisfying $\lrabs{{f}(x)-t(x)}\leq \eta$ $\forall{(x,y)\in S}$,
it holds that $\hP_{x\sim\cD}\lrset{x: \lrabs{f(x)-t(x)} \leq \eta+\alpha}\geq 1-\epsilon$.
\end{theorem}

\begin{proof}[of \cref{thm:generalization-interpolation-cutoffs}]
 Let $\cF\subseteq [0,1]^{\cX}$ and let 
$$
\cH=\lrset{(x,y)\mapsto\lrabs{f(x)-y}:f\in\cF}.
$$
Define the function classes 
$$\cF_1= \lrset{(x,y)\mapsto \lrabs{f(x)-y}-\psi(x,y):f\in \cF},$$
and 
$$\cF_2= \lrset{(x,y)\mapsto \max\lrset{f(x,y),0}:f \in \cF_1}.$$

We claim that $\fat(\cH,\gamma)=\fat(\cF_1,\gamma)$. Take a set $S=\lrset{(x_1,y_1),\ldots,(x_m,y_m)}$ that is $\gamma$-shattered by $\cH$. 
There exists a witness $r = (r_1, \ldots,r_m) \in [0,1]^m$ such that for each $\sigma = (\sigma_1, \ldots, \sigma_m) \in \{-1, 1\}^m$ there is a function $h_{\sigma} \in \cH$ such that
\[
\forall i \in [m] \; 
\begin{cases}
	h_{\sigma}((x_i,y_i)) \geq r_i + \gamma, & \text{if $\sigma_i = 1$}\\
    h_{\sigma}((x_i,y_i)) \leq r_i - \gamma, & \text{if $\sigma_i = -1$.}
 \end{cases}
\]
The set $S$ is shattered by $\cF_1$
by taking $\Tilde{r}= \lr{r_1+\eta(x_1,y_1),\ldots,r_m+\eta(x_m,y_m)}$. Similarly, any set that is shattered by $\cF_1$ is also shattered by $\cH$.

The class $\cF_2$ consists of choosing a function from $\cF_1$ and computing its pointwise maximum with the constant function 0.
In general, for two function classes $\cG_1,\cG_2$, we can define the maximum aggregation class
\begin{equation*}
 \max(\cG_1,\cG_2)=\{x\mapsto\max\lrset{g_1(x),g_2(x)}:g_i \in \cG_i\},
\end{equation*}
\citet{kontorovich2021fat} showed that for any $\cG_1,\cG_2$
\begin{align*}
 \fat\lr{ \max(\cG_1,\cG_2),\gamma}   \lesssim
\lr{\fat\lr{\cG_1,\gamma}+\fat\lr{\cG_2,\gamma}}\log^2\lr{\fat\lr{\cG_1,\gamma}+\fat\lr{\cG_2,\gamma}}.
\end{align*}
Taking $\cG_1=\cF_1$ and $\cG_2\equiv 0$, we get 
\begin{align*}
 \fat\lr{\cF_2,\gamma}   \lesssim
\fat\lr{\cF_1,\gamma}\log^2\lr{\fat\lr{\cF_1,\gamma}}.
\end{align*}
For the particular case $\cG_2\equiv 0$, we can show a better bound of 
\begin{align*}
 \fat\lr{\cF_2,\gamma}   \lesssim
\fat\lr{\cF_1,\gamma}.
\end{align*}
In words, it means that truncation cannot increase the shattering dimension.
Indeed, take a set $S=\lrset{(x_1,y_1),\ldots,(x_k,y_k)}$ that is $\gamma$-shattered by $\cF_2=\max\lr{\cF_1,0}$, we show that this set is  $\gamma$-shattered by $\cF_1$.
There exists a witness $r = (r_1, \ldots,r_k) \in [0,1]^k$ such that for each $\sigma = (\sigma_1, \ldots, \sigma_k) \in \{-1, 1\}^k$ there is a function $f_{\sigma} \in \cF_1$ such that
\[
\forall i \in [k] \; 
\begin{cases}
    \max\lrset{f_{\sigma}((x_i,y_i)),0} \geq r_i + \gamma, & \text{if $\sigma_i = 1$}\\
    \max\lrset{f_{\sigma}((x_i,y_i)),0} \leq r_i - \gamma, & \text{if $\sigma_i = -1$.}
 \end{cases}
\]
For $\max\lrset{f_{\sigma}((x_i,y_i)),0} \leq r_i - \gamma$, we simply have that $f_{\sigma}((x_i,y_i)) \leq r_i - \gamma$. Moreover, this implies that $r_i\geq \gamma$. As a result, 
\begin{align*}
     \max\lrset{f_{\sigma}((x_i,y_i)),0} 
     &\geq 
     r_i + \gamma
     \\
     &\geq
     2\gamma
    \\
     &>
     0,
\end{align*}
which means that $f_{\sigma}((x_i,y_i)) \geq r_i + \gamma$. This shows that $\cF_1$ $\gamma$-shatters $S$ as well.
We can conclude the proof by applying \cref{thm:generalization-interpolation} to the class $\cF_2$ with $t(x)=0$ and $\eta=\alpha$.
\end{proof}


The following boosting and sparsification claims were proven for the case of a fixed cut-off parameter. The proofs extend similarly to the case of a changing cut-off parameter $\psi : \cX \times \cY \rightarrow [0,1]$.
\paragraph{Boosting.}
Following~\cite{hanneke2019sample}, we define the weighted median as 
\begin{equation*}
    \med\lr{y_1,\ldots,y_T;\alpha_1,\ldots,\alpha_T} 
    = 
    \min\lrset{y_j: \frac{\sum^T_{t=1}\alpha_t\hI\lrbra{y_j<y_t}}{\sum^T_{t=1}\alpha_t}<\frac{1}{2}} \,,
\end{equation*}
and the weighted quantiles, for $\beta \in [0,1/2]$, as
\begin{align*}
    Q^+_\beta(y_1,\ldots,y_T;\alpha_1,\ldots,\alpha_T) = \min\lrset{y_j: \frac{\sum^T_{t=1}\alpha_t\hI\lrbra{y_j<y_t}}{\sum^T_{t=1}\alpha_t}<\frac{1}{2} - \beta} \\
     Q^-_\beta(y_1,\ldots,y_T;\alpha_1,\ldots,\alpha_T) = \max\lrset{y_j: \frac{\sum^T_{t=1}\alpha_t\hI\lrbra{y_j>y_t}}{\sum^T_{t=1}\alpha_t}<\frac{1}{2} - \beta}.
\end{align*}
We define $Q^+_\beta(f_1,\ldots,f_T;\alpha_1,\ldots,\alpha_T)(x) = Q^+_\beta(f_1(x),\ldots,f_T(x);\alpha_1,\ldots,\alpha_T)$, and 
$Q^-_\beta(f_1,\ldots,f_T;\alpha_1,\ldots,\alpha_T)(x) = Q^-_\beta(f_1(x),\ldots,f_T(x);\alpha_1,\ldots,\alpha_T)$, 
and $\med(f_1,\ldots,f_T;\alpha_1,\ldots,\alpha_T)(x) = \med(f_1(x),\ldots,f_T(x);\alpha_1,\ldots,\alpha_T)$.
We omit the weights $\alpha_i$ when they are equal to each other.
The following guarantee holds for the boosting procedure.
\begin{lemma}\label{lem:boosting}
Let $S=\{(x_i,y_i)\}^m_{i=1}$, $T=O\lr{\frac{1}{\beta^2}\log(m)}$. Let $\hat{f}_1,\ldots,\hat{f}_T$ and $\alpha_1,\ldots,\alpha_T$ be the functions and coefficients returned from the median boosting procedure with changing cut-offs (Step 2 in \Cref{alg:compression}). For any $i\in\lrset{1,\ldots,m}$ it holds that
\begin{equation*}
  \max\lrset{\lrabs{Q^+_{\beta/2}(\hat{f}_1,\ldots,\hat{f}_T;\alpha_1,\ldots,\alpha_T))(x_i)- y_i}, \lrabs{Q^-_{\beta/2}\hat{f}_1,\ldots,\hat{f}_T;\alpha_1,\ldots,\alpha_T)(x_i) - y_i}} \leq \psi(x,y)+2\alpha.
\end{equation*}
\end{lemma}

\paragraph{Sparsification.} 
\begin{lemma}\label{lem:sparsification}
    Choosing 
    $$n=\Theta\lr{\frac{1}{\beta^2}\fat^*\lr{\cF,c\alpha}\log^2\lr{\fat^*\lr{\cF,c\alpha}/\alpha}}$$ 
    in Step 3 of \Cref{alg:compression}, we have for all $(x,y)\in S$
    $
     \lrabs{\med\lr{f_{J_1}(x),\ldots,f_{J_n}(x)}-y} \leq \psi(x,y)+3\alpha,
    $
    where $c>0$ is a universal constant.
\end{lemma}

\end{document}